\documentclass[A4]{article}
\usepackage{lipsum}
\usepackage{amsfonts,amssymb,amsmath,amsthm}
\usepackage{bbm}
\usepackage{graphicx}
\usepackage{epstopdf}
\usepackage{algorithmic}
\usepackage{enumitem}
\setlist[enumerate]{leftmargin=.5in}
\setlist[itemize]{leftmargin=.5in}
\usepackage{amsopn}
 \usepackage{enumitem}   
\usepackage{multirow}
\usepackage{subcaption}
\usepackage[export]{adjustbox}
\usepackage{pgfplots}
\usetikzlibrary{math}

\newcommand{\RN}[1]{%
  \textup{\uppercase\expandafter{\romannumeral#1}}%
}

\ifpdf
  \DeclareGraphicsExtensions{.eps,.pdf,.png,.jpg}
\else
  \DeclareGraphicsExtensions{.eps}
\fi

\def\V#1{{\ifcat#1\relax \boldsymbol{#1} \else \bf{#1} \fi} }   

\newcommand{\eqdef}{\stackrel{\vartriangle}{=}}
\newtheorem{remark}{Remark}

\newtheorem{definition}{Definition}
\newtheorem{theorem}{Theorem}
\newtheorem{proposition}{Proposition}
\newtheorem{lemma}{Lemma}
\title{Measuring Complexity of Learning Schemes Using  Hessian-Schatten Total Variation}

\author{Shayan Aziznejad\thanks{ Biomedical Imaging Group, EPFL, Lausanne, Switzerland (shayan.aziznejad@epfl.ch, joaquim.campos@epfl.ch,   michael.unser@epfl.ch) This work was supported in part by the European Research Council (ERC Project FunLearn) under Grant 101020573 and in part by the Swiss National Science Foundation, Grant 200020\_184646/1. }
\and Joaquim Campos \footnotemark[1] \and Michael Unser\footnotemark[1]  }

\begin{document}

\maketitle
\begin{abstract}
In this paper, we introduce the Hessian-Schatten total variation (HTV)---a novel seminorm that quantifies the total ``rugosity'' of multivariate functions. Our motivation for defining HTV is to  assess  the complexity of supervised-learning schemes. We start by specifying the adequate matrix-valued Banach spaces that are equipped with suitable classes of mixed norms. We then show that the HTV is invariant to rotations, scalings, and translations. Additionally, its minimum value is achieved for linear mappings,  which supports the common intuition that linear regression is the least complex learning model.  Next, we present closed-form expressions of the HTV for two general classes of functions. The first one is the class of Sobolev functions with a certain degree of regularity, for which we show that the HTV coincides with the Hessian-Schatten seminorm that is sometimes used as a regularizer for image reconstruction. The second one is the class of  continuous and piecewise-linear (CPWL) functions. In this case, we show that the HTV reflects the total change in slopes between linear regions that have a common facet. Hence, it can be viewed as a convex relaxation ($\ell_1$-type) of the number of linear regions ($\ell_0$-type) of CPWL mappings.   Finally, we illustrate the use of  our proposed seminorm.
 \end{abstract}
\textbf{Key words:}  Hessian operator, Schatten norm, total variation, continuous and piecewise-linear functions, supervised learning.

\section{Introduction} 
Given the sequence $(\boldsymbol{x}_m,y_m)\in \mathbb{R}^d\times \mathbb{R},m=1,\ldots,M$ of data points, the goal of supervised learning is to construct a mapping $f:\mathbb{R}^d\rightarrow \mathbb{R}$ that adequately explains the data, {\it i.e.} $f(\boldsymbol{x}_m)\approx y_m$, while avoiding the problem of overfitting \cite{wahba1990spline,gyorfi2006distribution,hastie2009overview}. This is often formulated as a minimization problem of the  form
\begin{equation}\label{Eq:Learning}
\min_{f\in \mathcal{F}} \left( \sum_{m=1}^M E\left(f(\boldsymbol{x}_m),y_m\right) + \lambda \mathcal{R}(f)\right),
\end{equation}
where $\mathcal{F}$ is the search space, $E:\mathbb{R}\times\mathbb{R}$ is a loss function that quantifies data discrepancy, and $\mathcal{R}:\mathcal{F}\rightarrow \mathbb{R}$ is a functional that enforces regularization. The regularization parameter $\lambda>0$  adjusts the contribution of the two terms. A classical example is learning over reproducing-kernel Hilbert spaces (RKHS), where $\mathcal{F}=\mathcal{H}(\mathbb{R}^d)$ is an RKHS and $\mathcal{R}(f)=\|f\|_{\mathcal{H}}^2$ \cite{poggio1990networks,poggio1990regularization}.  The key result in this framework is the kernel representer theorem that  provides a parametric form for the learned mapping \cite{kimeldorf1971some,scholkopf2001generalized}. This foundational result is at the heart of many kernel-based schemes, such as support-vector machines  \cite{scholkopf2001learning,evgeniou2000regularization,steinwart2008support}. 
Moreover,  there has been an interesting line of works regarding the statistical optimality of kernel-based methods \cite{caponnetto2007optimal,mendelson2010regularization,rakhlin2019consistency,steinwart2009optimal}. 
A central element in these analyses is that the regularization functional $\mathcal{R}(\cdot)$ (in this case, the underlying Hilbertian norm) directly controls the complexity of the learned mapping \cite[Section 2.4]{bartlett2021deep}.

Although kernel methods are supported by a sound theory, they have been outperformed by deep neural networks (DNNs) in various areas of application \cite{lecun2015deep,goodfellow2016deep}.  DNN-based methods are the current state of the art in several image processing tasks, such as inverse problems \cite{Jin2017deep}, image classification  \cite{krizhevsky2012imagenet},  and image segmentation \cite{Ronneberger2015}. Unlike kernel methods, DNNs have   intricate nonlinear  structures and the reason of their outstanding performance  is not yet fully understood \cite{bartlett2021deep}.   A possible approach to the comparison of DNNs is to quantify the ``complexity'' of the learned mapping. For example, neural networks with rectified linear units (ReLU),  ${\rm ReLU}(x) = \max(x,0)$ \cite{Glorot2011}, are known to produce continuous and piecewise-linear (CPWL) mappings. Consequently, the number of linear regions of the input-output mapping has been proposed as  a measure of complexity  in this case \cite{pascanu2013number,hanin2019complexity}. While this is an interesting metric to study, it has two limitations. The first is that this quantity is only defined for CPWL functions and, consequently, only applicable to ReLU neural networks. This prevents one from building a framework that would include neural networks with more modern activation functions  \cite{clevert2015fast,hendrycks2016gaussian,ramachandran2017searching,misra2019mish}. The second limitation is that this measure is not robust, in the sense that the input-output mapping might have many small regions around the training data points and still be able to generalize well.  This phenomena, which is called ``benign overfitting'' \cite{bartlett2020benign,li2021towards}, cannot be reflected in the aforementioned complexity measure.

In this paper, we introduce a novel seminorm---the Hessian-Schatten total variation (HTV)---and we propose its use as  a way to quantify the complexity of learning schemes.  Our definition of the HTV is based on a second-order extension of the space of functions with bounded variation \cite{ambrosio2000functions}.   We show that the HTV seminorm  satisfies the following desirable properties: 
\begin{enumerate}
\item It assigns the zero value for linear regression, which is the simplest learning scheme.   
\item  It is invariant (up to a multiplicative factor) to simple transformations (such as linear isometries and scaling) over the input domain.  
\item It is defined for both smooth and CPWL functions. Hence, it is applicable to a broad class of learning schemes, including ReLU neural networks and radial-basis functions.
\item It  favors CPWL functions with a small number of linear regions, thus promoting  a simpler (and, hence, more interpretable) representation of the data (Occam's razor principle). 
\end{enumerate}
 We provide closed-form formulas for the HTV of both smooth and CPWL functions. For smooth functions, the HTV  coincides with the Hessian-Schatten seminorm which is often used as a regularization term in linear inverse problems \cite{Lefki2012Hessian,Lefki2013Hessian}.  For CPWL functions, the HTV is a convex relaxation of the number of linear regions. This is analogous to the classical $\ell_0$ penalty in the field of compressed sensing,  where it is often replaced by its convex proxy, the $\ell_1$ norm,  to ensure tractability \cite{donoho2006compressed,eldar2012compressed}.  
  
The paper is organized as follows: We start Section \ref{sec:prelim} with some mathematical preliminaries that are essential for this paper. In Section \ref{sec:HTV}, we introduce the HTV seminorm and  prove its desirable properties. We then compute the HTV of two general classes of functions (smooth and CPWL) in Section \ref{sec:closed}. Finally,  we illustrate the practical aspects of our proposed seminorm with  examples in Section \ref{sec:numerical}. 
\section{Preliminaries}\label{sec:prelim}
  Throughout the paper, we denote the input domain by $\Omega\subseteq\mathbb{R}^d$.  Throughout the paper, we assume $\Omega$ to be an open ball of radius $R>0$,  with the convention that the case $R=+\infty$  corresponds to $\Omega=\mathbb{R}^d$. 

 \subsection{Schatten Matrix Norms}
For any $p\in[1,+\infty]$,   the Schatten-$p$ norm of a   real-valued matrix ${\bf A}\in\mathbb{R}^{d\times d}$ is defined as
\begin{equation}\label{Eq:SpNormDef} 
\|{\bf A}\|_{S_p} \eqdef \begin{cases}\left(\sum_{i=1}^d |\sigma_i({\bf A})|^p\right)^{\frac{1}{p}}, & 1\leq p<+\infty\\ \max_i |\sigma_i({\bf A})|, & p=+\infty,\end{cases}
\end{equation}
where $\left(\sigma_1({\bf A}),\ldots,\sigma_d({\bf A})\right)$ are the singular values of ${\bf A}$  \cite{bhatia2013matrix}. It is known that  the dual of the Schatten-$p$ norm is the Schatten-$q$ norm, where $q\in [1,\infty]$ is the H\"older conjugate of $p$ such that  $\frac{1}{p}+\frac{1}{q}=1$.  This result stems from   a variant of the H\"older  inequality for Schatten norms.  It states that 
\begin{equation}\label{Eq:Holder}
\langle {\bf A}, {\bf B}\rangle \eqdef \mathrm{Tr}\left( {\bf A}^T {\bf B} \right) \leq \|{\bf A}\|_{S_p} \|{\bf B}\|_{S_q} 
\end{equation}
for any pair of matrices ${\bf A}, {\bf B} \in \mathbb{R}^{d\times d}$ (see \cite{Lefki2015StructureTensor} for a simple proof). 
\subsection{Total-Variation Norm} 
\label{subsec:saclarspac}
  Schwartz' space of infinitely differentiable and compactly supported test functions $\varphi:\Omega \rightarrow \mathbb{R}$ is denoted by $\mathcal{D}(\Omega)$. Its continuous dual  $\mathcal{D}'(\Omega)$ is the space of   distributions  \cite{schwartz1957theorie}.   The Banach space $\mathcal{C}_0(\Omega)$  is the   completion of $\mathcal{D}(\Omega)$ with respect to the $L_\infty$ norm $\|f\|_{L_\infty} \eqdef \sup_{\boldsymbol{x}\in\Omega} |f(\boldsymbol{x})|$.  The bottom line is that the space $\mathcal{C}_0(\Omega)$ is formed of  continuous functions $f:\Omega \rightarrow\mathbb{R}$ that vanish at infinity. The Riesz-Markov theorem states that  the dual  of $\mathcal{C}_0(\Omega)$ is the space $\mathcal{M}(\Omega)=\left( \mathcal{C}_0(\Omega)\right)'$  of bounded Radon measures equipped with the total-variation norm \cite{rudin2006real}
\begin{equation}\label{Eq:TVnorm}
\|w\|_{\mathcal{M}}\eqdef \sup_{\varphi\in \mathcal{D}(\Omega)\backslash \{0\}} \frac{\langle w,\varphi\rangle}{\|\varphi\|_{\infty}}.
\end{equation}
The space $\mathcal{M}(\Omega)$  is a superset of    the space $L_1(\Omega)$ of absolutely integrable  measurable functions  with $\|f\|_{\mathcal{M}} = \|f\|_{L_1}$ for any $f\in L_1(\Omega)$. Moreover, it contains   shifted Dirac impulses with $\|\delta(\cdot - \boldsymbol{x}_0) \|_{\mathcal{M}}= 1$ for any $\boldsymbol{x}_0\in\Omega$. The latter can be generalized to any distribution of the form $w_{\boldsymbol{a}} = \sum_{n\in\mathbb{Z}} a_n \delta(\cdot-\boldsymbol{x}_n)$ with $\|w_{\boldsymbol{a}}\|_{\mathcal{M}} = \|\boldsymbol{a}\|_{\ell_1}$ for any $ \boldsymbol{a}=(a_n)\in \ell_1(\mathbb{Z})$ and any sequence of distinct locations $(\boldsymbol{x}_n)\subseteq \Omega$. 
\subsection{Matrix-Valued Banach Spaces}\label{Subsec:MValuedSpc}
 In this work, we are interested in the matrix-valued extension of the spaces defined in Section \ref{subsec:saclarspac}. We denote by $\mathcal{C}_0(\Omega;\mathbb{R}^{d\times d})$ the space of continuous matrix-valued functions  $\mathbf{F}: \Omega \rightarrow \mathbb{R}^{d\times d}$  that vanish at infinity so that  $\lim_{\|\boldsymbol{x}\|\rightarrow\infty} \|\mathbf{F}(\boldsymbol{x}) \|= 0$ whenever the domain is unbounded. (Note that this definition does not depend on the choice of the norms,  because they are all equivalent in finite-dimensional vector spaces.) Any matrix-valued function $\mathbf{F} :\Omega \rightarrow \mathbb{R}^{d\times d}$ has the unique representation 
\begin{equation} \label{Eq:MatrixRep}
 \mathbf{F}= [f_{i,j}]=  \begin{pmatrix}
 f_{1,1} & \cdots & f_{1,d}  \\
\vdots & \ddots & \vdots \\
 f_{d,1} & \cdots & f_{d,d} 
 \end{pmatrix}, 
\end{equation}
where each entry $f_{i,j}:\Omega\rightarrow \mathbb{R}$ is a scalar-valued function for $i,j=1,\ldots,d$. In this representation, the space $\mathcal{C}_0(\Omega;\mathbb{R}^{d\times d})$ is the collection of matrix-valued functions of the form \eqref{Eq:MatrixRep} with $f_{i,j} \in \mathcal{C}_0(\Omega)$.
\begin{definition} \label{Def:LinfSinf}
Let $q \in [1,+\infty]$. For any $\mathbf{F}\in\mathcal{C}_0(\Omega;\mathbb{R}^{d\times d})$, the $L_{\infty}$-$S_{q}$  mixed norm  is defined as
\begin{equation}\label{Eq:LinfSinf}
\|\mathbf{F}\|_{L_\infty,S_{q}} \eqdef \left\|  \begin{pmatrix}
 \|f_{1,1}\|_{L_\infty} & \cdots & \|f_{1,d}\|_{L_\infty}  \\
\vdots & \ddots & \vdots \\
 \|f_{d,1}\|_{L_\infty} & \cdots & \|f_{d,d}\|_{L_\infty} 
 \end{pmatrix} \right\|_{S_{q}}. 
\end{equation}
\end{definition}
\begin{remark}
 In Definition \ref{Def:LinfSinf}, the $S_q$-norm appears as the outer norm. We remain faithful to this convention throughout the paper and always denote mixed norms in order of appearance, where the first  is the inner-norm and the second the outer-norm. 
\end{remark} 
Following \cite{unser2020directsum}, we deduce that  $\left(\mathcal{C}_0(\Omega;\mathbb{R}^{d\times d}),\|\cdot\|_{L_\infty,S_{q}}\right)$ is a {\it bona fide} Banach space, whose dual is  $(\mathcal{M}  (\Omega,\mathbb{R}^{d\times d}),\|\cdot\|_{\mathcal{M},S_p})$, where $\mathcal{M}(\Omega;\mathbb{R}^{d\times d})$   is  the collection of matrix-valued Radon measures of the form 
\begin{equation} 
 \mathbf{W}=  [w_{i,j}]=\begin{pmatrix}
 w_{1,1} & \cdots & w_{1,d}  \\
\vdots & \ddots & \vdots \\
 w_{d,1} & \cdots & w_{d,d} 
 \end{pmatrix},  \quad w_{i,j}\in \mathcal{M}(\Omega) \quad \forall i,j=1,\ldots,d,
\end{equation}
and the mixed $\mathcal{M}-S_p$ norm is defined as  
\begin{equation} \label{Eq:MS1}
\|\mathbf{W}\|_{\mathcal{M},S_p} \eqdef \left\|  \begin{pmatrix}
 \|w_{1,1}\|_{\mathcal{M}} & \cdots & \|w_{1,d}\|_{\mathcal{M}}  \\
\vdots & \ddots & \vdots \\
 \|w_{d,1}\|_{\mathcal{M}} & \cdots & \|w_{d,d}\|_{\mathcal{M}} 
 \end{pmatrix} \right\|_{S_p}.
\end{equation}
The duality product $\langle \cdot,\cdot\rangle : \mathcal{M}(\Omega;\mathbb{R}^{d\times d})\times \mathcal{C}_0(\Omega;\mathbb{R}^{d\times d})\rightarrow \mathbb{R}$ is then defined as 
\begin{equation}
\langle\mathbf{W},\mathbf{F}\rangle \eqdef \sum_{i=1}^d \sum_{j=1}^d \langle w_{i,j} , f_{i,j} \rangle. 
\end{equation}
Finally, we denote by $L_1(\Omega;\mathbb{R}^{d\times d})$, $\mathcal{D}(\Omega;\mathbb{R}^{d\times d})$,  and $\mathcal{D}'(\Omega;\mathbb{R}^{d\times d})$, the matrix-valued generalizations of the spaces  $L_1(\Omega)$, $\mathcal{D}(\Omega)$ and $\mathcal{D}'(\Omega)$, respectively.  

\subsection{Generalized Hessian Operator}
The operators $\frac{\partial^2 f}{\partial x_i\partial x_j}:\mathcal{D}'(\Omega)\rightarrow \mathcal{D}'(\Omega)$ are viewed as second-order weak partial derivatives. More precisely, for any $i,j=1,\ldots,d$ and any $w\in \mathcal{D}'(\Omega)$ , the tempered distribution $\frac{\partial^2 w}{\partial x_i \partial x_j}\{w\}\in\mathcal{D}'(\Omega)$   is defined as
$$ \left\langle\frac{\partial^2 }{\partial x_i \partial x_j} w , \varphi \right\rangle=\left\langle w ,\frac{\partial^2  }{\partial x_i \partial x_j}\varphi\right\rangle,$$
for all test functions $\varphi \in \mathcal{D}(\Omega)$. This leads to the following definition of the generalized Hessian operator over the space of tempered distributions. 
\begin{definition}
The Hessian operator $\mathrm{H}:\mathcal{D}'(\Omega)\rightarrow\mathcal{D}'(\Omega;\mathbb{R}^{d\times d})$ is defined as 
\begin{equation}
\mathrm{H}\{f \}= \begin{pmatrix} \frac{\partial^2 f}{\partial x_1^2} & \cdots & \frac{\partial^2 f}{\partial x_1 \partial x_d} \\ \vdots & \ddots & \vdots \\ \frac{\partial^2 f}{\partial x_d \partial x_1}  & \cdots & \frac{\partial^2 f}{\partial x_d^2}   \end{pmatrix}.
\end{equation}
\end{definition}

\section{The Hessian-Schatten Total Variation}\label{sec:HTV}
In order to properly define the HTV seminorm, we start by introducing a novel class of mixed norms over $\mathcal{C}_0(\Omega;\mathbb{R}^{d\times d})$.
\begin{definition} \label{Def:SinfLinf}
Let $q\in[1,+\infty]$. For any $\mathbf{F}\in\mathcal{C}_0(\Omega;\mathbb{R}^{d\times d})$, the $S_{q}-L_{\infty}$ mixed-norm  is defined as
\begin{equation}\label{Eq:SinfLinf}
\|\mathbf{F}\|_{S_{q},L_\infty} =\sup_{\boldsymbol{x}\in\Omega} \|\mathbf{F}(\boldsymbol{x})\|_{S_{q}}. 
\end{equation}
\end{definition}
In Section \ref{Subsec:MValuedSpc},  we highlighted that the dual norm of $L_{\infty}-S_q$ mixed-norm is $\mathcal{M}-S_p$, which is defined over matrix-valued Radon measures.  In Definition \ref{Def:SinfLinf}, we switched the order of  application of the individual norms; however, the two norms induce the same topology over the space   $\mathcal{C}_0(\Omega;\mathbb{R}^{d\times d})$.
\begin{theorem}\label{Thm:SinfLinf}
Regarding the mixed norms defined in Definitions \ref{Def:LinfSinf} and \ref{Def:SinfLinf}
\begin{enumerate}
\item The functional $\mathbf{F}\mapsto \|\mathbf{F}\|_{S_{q},L_\infty}$ is a well-defined (finite) norm over  $\mathcal{C}_0(\Omega;\mathbb{R}^{d\times d})$. 
\item The $L_{\infty}-S_{q}$  and  the $S_{q}-L_{\infty}$ mixed norms are equivalent,  in the sense that there exists positive constants $A,B>0$ such that,  for all ${\bf F}\in \mathcal{C}_0(\Omega;\mathbb{R}^{d\times d})$, we have that 
\begin{equation}\label{Ineq:NormEqv}
A \|{\bf F}\|_{S_q,L_\infty} \leq  \|{\bf F}\|_{L_\infty,S_q}  \leq B \|{\bf F}\|_{S_q,L_\infty}. 
\end{equation}
\item The normed space  $\left(\mathcal{C}_0(\Omega;\mathbb{R}^{d\times d}),\|\cdot\|_{S_{q},L_\infty}\right)$    is a   {\it bona fide} Banach space.
\end{enumerate}
\end{theorem}
The proof is available in Appendix \ref{App:SinfLinf}. Using the outcomes of Theorem \ref{Thm:SinfLinf} and, in particular, Item 3, we are now ready to introduce the $S_p-\mathcal{M}$ mixed norm   defined over the space of matrix-valued Radon measures. 
\begin{definition}\label{Def:RadonSchatten} 
For any matrix-valued Radon measure $\mathbf{W}\in \mathcal{M}  (\Omega,\mathbb{R}^{d\times d})$, the $S_p-\mathcal{M}$ mixed-norm is defined as 
\begin{equation}\label{Eq:RadonSchattenNorm}
\|\mathbf{W}\|_{S_p,\mathcal{M}} \eqdef  \sup \left\{ \langle \mathbf{W},\mathbf{F}\rangle:  \mathbf{F}\in  \mathcal{C}_0(\Omega;\mathbb{R}^{d\times d}), \|\mathbf{F}\|_{S_{q},L_\infty}=1 \right\}.
\end{equation}
\end{definition}

Intuitively, the $S_p-\mathcal{M}$ norm of a matrix-valued function ${\bf F}:\Omega\rightarrow\mathbb{R}^{d\times d}$ is equal to the total-variation norm of the function $\boldsymbol{x}\mapsto \|{\bf F}(\boldsymbol{x})\|_{S_q}$. However, this intuition cannot directly lead to a general definition because the space $\mathcal{M}(\Omega;\mathbb{R}^{d\times d})$ contains elements that do not have a pointwise definition. We are therefore forced to  define this norm by duality,  as opposed to the $ \mathcal{M}-S_p$ norm given in \eqref{Eq:MS1}. 

We also remark that, due to the dense embedding $\mathcal{D}(\Omega;\mathbb{R}^{d\times d})\hookrightarrow \mathcal{C}_0(\Omega;\mathbb{R}^{d\times d})$, one can alternatively express the $S_p-\mathcal{M}$ norm as 
\begin{equation} 
\|\mathbf{W}\|_{S_p,\mathcal{M}} = \sup \left\{ \langle \mathbf{W},\mathbf{F}\rangle:  \mathbf{F}\in  \mathcal{D}(\Omega;\mathbb{R}^{d\times d}), \|\mathbf{F}\|_{S_{q},L_\infty}=1 \right\},
\end{equation}
which is well-defined for all matrix-valued tempered distributions. However, the only elements of $\mathcal{D}'(\Omega;\mathbb{R}^{d\times d})$ of finite $S_p-\mathcal{M}$  norm are precisely the matrix-valued finite Radon measures. In other words, $\mathcal{M}  (\Omega,\mathbb{R}^{d\times d})$ is the largest subspace of $\mathcal{D}'(\Omega;\mathbb{R}^{d\times d})$ with finite $S_p-\mathcal{M}$ norm.

 In what follows, we strengthen the intuition behind the $S_p-\mathcal{M}$ norm by computing  it for two general classes of functions/distributions in $\mathcal{M}  (\Omega,\mathbb{R}^{d\times d})$ that are particularly important in our framework: the absolutely integrable matrix-valued functions and the Dirac fence distributions. 
\begin{definition}\label{Def:DiracFence}
For any nonzero matrix $\mathbf{A}\in\mathbb{R}^{d\times d}$, any convex compact set $C \subset\mathbb{R}^{d_1}$ with $d_1<d$, and any measurable transformation $\mathbf{T}:\mathbb{R}^{d_1}\rightarrow\mathbb{R}^{d-d_1}$ (not necessarily linear) such that $C \times \mathbf{T}(C) \subseteq \Omega$, we define the corresponding Dirac fence $\mathbf{D}\in \mathcal{M}(\Omega;\mathbb{R}^{d\times d})$ as 
\begin{equation}\label{Eq:WformDirac}
\mathbf{D}(\boldsymbol{x}_1,\boldsymbol{x}_2) = \mathbf{A}\mathbbm{1}_{\boldsymbol{x}_1\in C}\delta(\boldsymbol{x}_2 - \mathbf{T}\boldsymbol{x}_1\}), \quad \boldsymbol{x}_1 \in \mathbb{R}^{d_1}, \boldsymbol{x}_2 \in \mathbb{R}^{d-d_1}, (\boldsymbol{x}_1,\boldsymbol{x}_2) \in \Omega.
\end{equation}
\end{definition}
  Dirac fence distributions are natural generalizations of  {\it the Dirac impulse} to nonlinear (and bounded) manifolds \cite{onural2006impulse}. More precisely, for any test function $\mathbf{F}\in \mathcal{C}_0(\Omega;\mathbb{R}^{d\times d})$ and any Dirac fence ${\bf D}$ of the form \eqref{Eq:WformDirac}, we have that 
\begin{equation}
\langle {\bf D} , {\bf F} \rangle = \int_{C} {\rm Tr}\left({\bf A}^T{\bf F}(\boldsymbol{x}_1,{\bf T}\boldsymbol{x}_1)\right) {\rm d}\boldsymbol{x}_1 \in \mathbb{R}.
\end{equation}
Intuitively, this corresponds to considering a ``continuum'' of low-dimensional Dirac impulses on the $d_1$-dimensional compact manifold $C \times \mathbf{T}(C)$  that is embedded in $\Omega$, as illustrated in Figure \ref{Fig:DiracDist}.

\begin{figure}[t]
\begin{minipage}{1.0\linewidth}
  \centering
  \centerline{\includegraphics[width=0.7\linewidth]{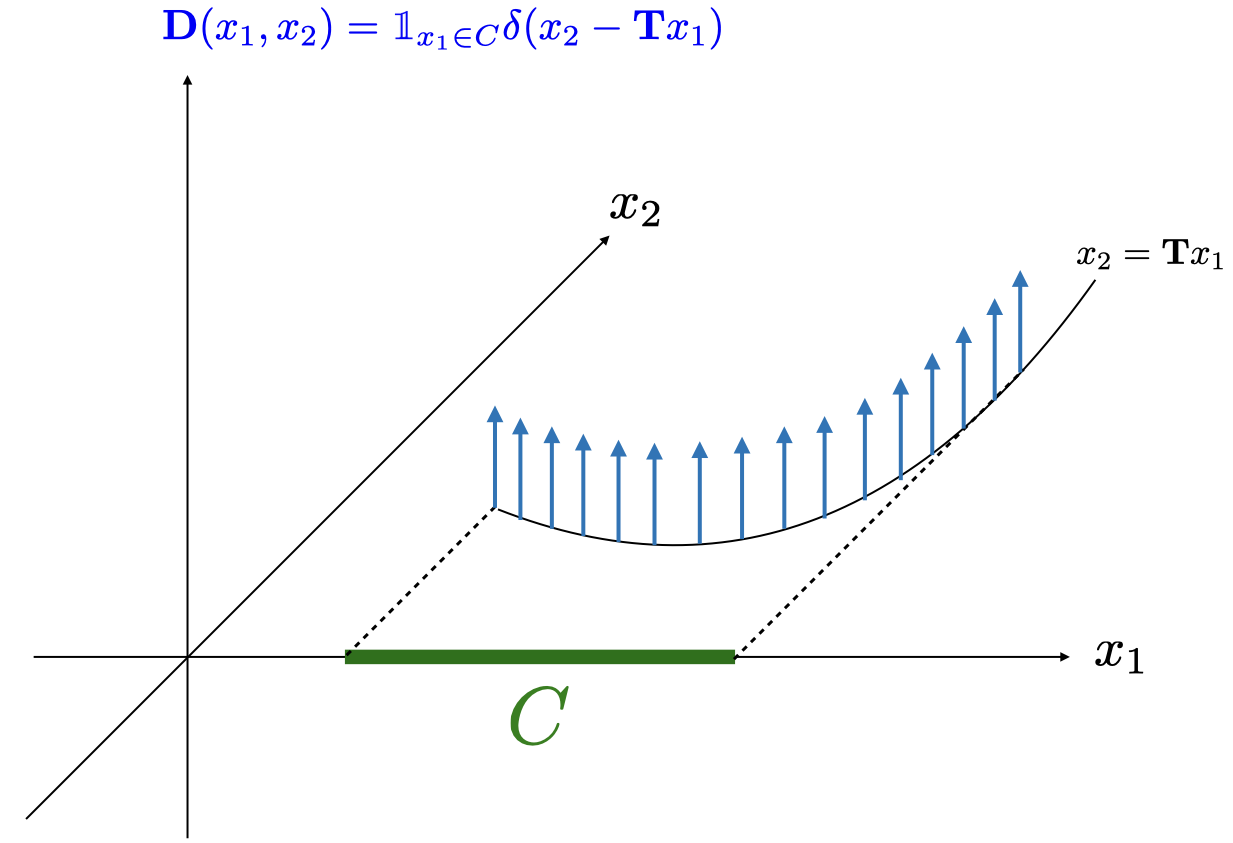}}
  \caption{  Illustration of a Dirac fence with $d=2$ and $d_1=1$.}
  \label{Fig:DiracDist} \medskip
\end{minipage}
\end{figure}

\begin{theorem}\label{Thm:RS1} 
Let $p\in[1,+\infty)$. 
\begin{enumerate}
\item   For any  matrix-valued function $\mathbf{W} \in L_1(\Omega,\mathbb{R}^{d\times d})\subseteq \mathcal{M}(\Omega;\mathbb{R}^{d\times d})$,  we have that
\begin{equation}\label{Eq:DualNormEqForm}
\|\mathbf{W}\|_{S_p,\mathcal{M}}=  \big\| \left\|\mathbf{W}(\cdot)\right\|_{S_p}  \big\|_{L_1}= \int_{\Omega} \left(\sum_{i=1}^d \left|\sigma_i\left( \mathbf{W}(\boldsymbol{x})\right)\right|^p \right)^{\frac{1}{p}}\mathrm{d}\boldsymbol{x}.
\end{equation}
\item  For any Dirac fence distribution ${\bf D}$ of the form \eqref{Eq:WformDirac}, we have that     
\begin{equation}
\|\mathbf{D} \|_{S_p,\mathcal{M}} = \|\mathbf{A}\|_{S_p} \mathrm{Leb}(C),
\end{equation}
where $\mathrm{Leb}(C)$ denotes the Lebesgue measure of $C\subseteq \mathbb{R}^{d_1}$.
\item Consider two Dirac fences ${\bf D}_1$ and ${\bf D}_2$ of the form 
$$\mathbf{D}_i(\boldsymbol{x}_1,\boldsymbol{x}_2) = \mathbf{A}_i\mathbbm{1}_{\boldsymbol{x}_1\in C_i}\delta(\boldsymbol{x}_2 - \mathbf{T}_i\boldsymbol{x}_1\}), \quad i=1,2$$
and assume that the ``intersection'' of the two fences is of measure zero,  in the sense that $C_0=\{\boldsymbol{x}_1\in C_1 \cap C_2: {\bf T}_1\boldsymbol{x}_1={\bf T}_2\boldsymbol{x}_1\}$ is a  subset of $\mathbb{R}^{d_1}$ whose Lebesgue measure is zero.  Then, we have that 
\begin{equation}
\|\mathbf{D}_1+\mathbf{D}_2 \|_{S_p,\mathcal{M}} =\|\mathbf{D}_1\|_{S_p,\mathcal{M}}+ \|\mathbf{D}_2 \|_{S_p,\mathcal{M}}.
\end{equation}
\end{enumerate}
\end{theorem}
The proof can be found in Appendix \ref{App:RS1}. We are now ready to define the HTV seminorm. 
 \begin{definition}\label{Def:RSH}
Let $p\in [1,+\infty]$. The Hessian-Schatten total variation  of any $f\in\mathcal{D}'(\Omega)$ is defined as 
\begin{equation}\label{Eq:RSH}
{\rm HTV}_p(f) = \|{\rm H}\{f\}\|_{S_p,\mathcal{M}} = \sup \left\{ \langle {\rm H}\{f\} ,\mathbf{F}\rangle:  \mathbf{F}\in  \mathcal{D}(\Omega;\mathbb{R}^{d\times d}), \|\mathbf{F}\|_{S_{q},L_\infty}=1 \right\},
\end{equation}
where $q\in [1,+\infty]$ is the H\"older conjugate of $p$ with $\frac{1}{p}+\frac{1}{q}=1$. 
 \end{definition}
We remark that the case $p=2$ has been previously studied in the context of the  space of functions with bounded Hessian 
\cite{demengel1984fonctions,hinterberger2006variational,bredies2010total,bergounioux2010second}. In our work, we complement their theoretical findings by extending the definition of the HTV to all Schatten norms with arbitrary value of $p\in [1,+\infty]$.  We now prove some desirable properties of the HTV functional. The proofs are to be found in Appendix \ref{App:AI}.
 \begin{theorem}\label{Thm:AI}
The HTV seminorm satisfies the following properties. 
\begin{enumerate}
 \item {\bf Null Space}: A tempered distribution has a vanishing HTV  if and only if  it can be identified as an affine function. In other words, we have that 
$$ \mathcal{N}_{\mathrm{HTV}_p}(\Omega) = \{f\in \mathcal{D}'(\Omega): {\rm HTV}_p(f) = 0\}= \{ \boldsymbol{x}\mapsto\boldsymbol{a}^T\boldsymbol{x} + b: \boldsymbol{a}\in\mathbb{R}^d, b\in\mathbb{R}\}.$$
 
\item {\bf Invariance}: Let $\Omega= \mathbb{R}^d$. For any $f\in \mathcal{D}'(\mathbb{R}^d)$, we have that 
\begin{align*}
&\mathrm{HTV}_p\left( f(\cdot - \boldsymbol{x}_0)  \right) = \mathrm{HTV}_p\left(f\right), & \forall \boldsymbol{x}_0 \in \mathbb{R}^d, \\ 
&\mathrm{HTV}_p\left( f(\alpha \cdot )  \right) = |\alpha|^{2-d} \mathrm{HTV}_p\left(f\right), & \forall \alpha\in \mathbb{R},\\
&\mathrm{HTV}_p\left( f({\bf U} \cdot )  \right) = \mathrm{HTV}_p\left(f\right),  &\forall {\bf U}\in \mathbb{R}^{d\times d}: \text{Orthonormal}.
\end{align*}
\end{enumerate} 
\end{theorem}

\section{Closed-Form Expressions for the HTV of Special Functions} \label{sec:closed}
Although Definition \ref{Def:RSH}  introduces a  formal way to compute the HTV of a given element $f\in \mathcal{D}'(\Omega)$, it is still very abstract and not practical. This is the reason why we now provide closed-form expressions for the HTV of two general classes of functions.
\subsection{Sobolev Functions} 
 Let $W_1^2(\Omega)$ be the Sobolev space of twice-differentiable functions $f:\Omega\rightarrow \mathbb{R}$ whose  second-order partial derivatives are in $L_1(\Omega)$. We note that, for compact domains $\Omega$, this space contains the input-output relation of neural networks with activation functions that are twice-differentiable  almost everywhere ({\it e.g.}, sigmoid \cite{cybenko1989approximation}, Swish \cite{ramachandran2017searching}, Mish \cite{misra2019mish},  GeLU \cite{hendrycks2016gaussian}). 
\begin{proposition}[{\bf Sobolev Compatibility}]\label{Prop:Sobolev}
Let $p\in [1,+\infty]$. Then,  for any Sobolev function $f\in W_1^2(\Omega)$, we have that
$$ \mathrm{HTV}_p(f) = \|{\rm H}\{f\}\|_{S_p,L_1} = \int_{\Omega} \|{\rm H}\{f\}(\boldsymbol{x})\|_{S_p} {\rm d}\boldsymbol{x}.$$
\end{proposition}
\begin{proof}
This is a consequence of Theorem \ref{Thm:RS1} since, for any $f\in W_1^2(\Omega)$, the matrix-valued function ${\rm H}\{f\}:\Omega\rightarrow\mathbb{R}^{d\times d}: \boldsymbol{x}\mapsto {\rm H}\{f\}(\boldsymbol{x})$ is measurable and is in $L_1(\Omega;\mathbb{R}^{d\times d})$. 
\end{proof}
Interestingly, Proposition \ref{Prop:Sobolev} demonstrates that our introduced seminorm is a generalization of the Hessian-Schatten regularization that has been used in inverse problems and image reconstruction \cite{Lefki2012Hessian,Lefki2013Hessian}. 

\subsection{Continuous and Piecewise-Linear Mappings}
A function $f:\Omega\rightarrow \mathbb{R}$ is said to be continuous and piecewise linear  if 
\begin{enumerate}
\item It is continuous. 
\item There exists a finite partitioning $\Omega = P_1 \sqcup P_2 \sqcup \cdots \sqcup P_N$ such that,  for any $n=1,\ldots,N$, $P_n$ is a convex polytope with the property that the restricted function $f\big|_{P_n}$ is an affine mapping of the form $f\big|_{P_n}(\boldsymbol{x})= \boldsymbol{a}_n^T\boldsymbol{x} + b_n$ for all $\boldsymbol{x} \in P_n$.
\end{enumerate}

An example of a CPWL function is shown in Figure \ref{Fig:CPWL1}. Let us highlight that there is an intimate link between CPWL functions and ReLU neural networks.  Indeed, it has been shown that the input-output relation of any feed forward ReLU neural network is a CPWL function \cite{pascanu2013number,Montufar2014}. Moreover, any CPWL function can  be represented {\it exactly} by some ReLU neural network \cite{arora2016understanding}.   
\begin{figure}[t]
\begin{minipage}{1.0\linewidth}
  \centering
  \centerline{\includegraphics[width=0.8\linewidth]{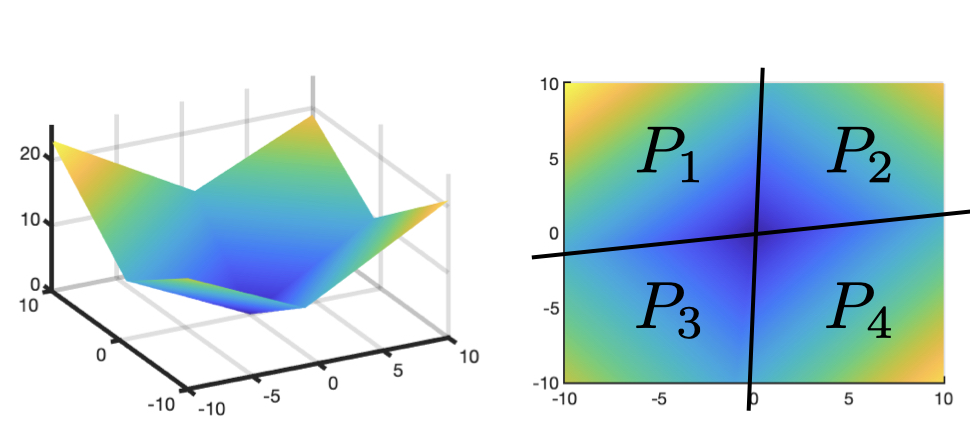}}
  \caption{  Illustration of a CPWL function $f:\mathbb{R}^2\rightarrow\mathbb{R}$. Left: 3D view. Right: 2D partitioning.}
  \label{Fig:CPWL1} \medskip
\end{minipage}
\end{figure}

\begin{theorem}\label{Thm:CPWLcalc}
Let $f:\Omega \rightarrow \mathbb{R}$ be the CPWL function  described above. For any $p\in[1,+\infty]$, the corresponding HTV of $f$ is given as 
\begin{equation}\label{Eq:CPWLcalc}
 \mathrm{HTV}_p(f) = \frac{1}{2}\sum_{n=1}^N \sum_{k\in {\rm adj}_n} \|\boldsymbol{a}_{n}-\boldsymbol{a}_{k}\|_2 H^{d-1}(P_n \cap P_k),
\end{equation}
where ${\rm adj}_n$ is the set of indices   $k\in\{1,\ldots,N\}$ such that $P_n$ and $P_k$ are neighbors and  $H^{d-1}$ denotes the $(d-1)$-dimensional Hausdorff measure.
\end{theorem}
The proof of Theorem \ref{Thm:CPWLcalc} is provided in Appendix \ref{App:CPWL}. We conclude from \eqref{Eq:CPWLcalc} that the HTV seminorm accounts for the change of (directional) slope in all the junctions in the partitioning. Specifically, the HTV of a CPWL function is proportional to a weighted $\ell_1$ penalty on the vector of slope changes, where the weights are proportional to the volume of the intersection region. This can be seen as a convex relaxation of the number of linear regions,.  The latter has the disadvantage  that is unable to differentiate between small and large changes of slope. Another noteworthy observation is the invariance of the HTV of CPWL functions to the value of $p\in[1,+\infty)$, which is unlike the case of Sobolev functions in Proposition \ref{Prop:Sobolev}. This is due to the extreme sparsity of the Hessian of CPWL functions. In fact, the Hessian matrix is zero everywhere except at the borders of linear regions. There, it is a Dirac fence weighted by a rank-1 matrix. The invariance  then follows from the observation that the Schatten-$p$ norms collapse to a single value in rank-1 matrices ({\it i.e.,} their only nonzero singular value). 

%
%
   \section{Illustrations of Usage}\label{sec:numerical}
In this section, we illustrate the behavior of the HTV seminorm in different scenarios.  The associated codes are available online\footnote{https://github.com/joaquimcampos/HTV-Learn}. In our first example, we consider the problem of learning one-dimensional mappings from noisy data. Let us mention that,  in dimension $d=1$, the HTV coincides with the second-order total-variation (TV-2)  seminorm, ${\rm TV}^{(2)}(f)= \|{\rm D}^2 \{f\}\|_{\mathcal{M}}$, which has been used to learn  activation functions of deep neural networks \cite{unser2019deepspline,bohra2020deepspline,aziznejad2020deepspline}. In this example, we compare five different learning schemes:
\begin{enumerate}
\item A ReLU neural network with three hidden layers, each layer consisting of 10 neurons;
\item CPWL learning using TV-2 regularization \cite{debarre2020sparsest};
\item  CPWL learning using Lipschitz regularization \cite{aziznejad2021Lipschitz};
\item CPWL learning  using the $L_2$ norm of the first derivative as the regularization term (smoothing spline);
\item RKHS learning with a Gaussian reproducing kernel $k(x,y) = \exp\left( - (x-y)^2/(2\sigma^2)\right)$ whose width is $\sigma = 1/4$.
\end{enumerate} 
We set the hyper parameters of each method such that they all have a similar training loss. The learned mappings are depicted in Figure \ref{Fig:1D}, where we have also indicated  their corresponding HTV  value.  As can be seen, the models that have a lower HTV are simpler and visually more satisfactory. Moreover, we observe that the neural network produces a CPWL mapping with similar complexity as the one produced by the  TV-2 regularization scheme, which is expected to yield the mapping with the smallest HTV. This is in line with the recent results in deep-learning theory that indicate the existence of certain implicit regularizations in the learning of neural networks \cite{savarese2019infinite}.
\begin{figure}[t]
\begin{minipage}{1.0\linewidth}
  \centering
  \centerline{\includegraphics[width=\linewidth]{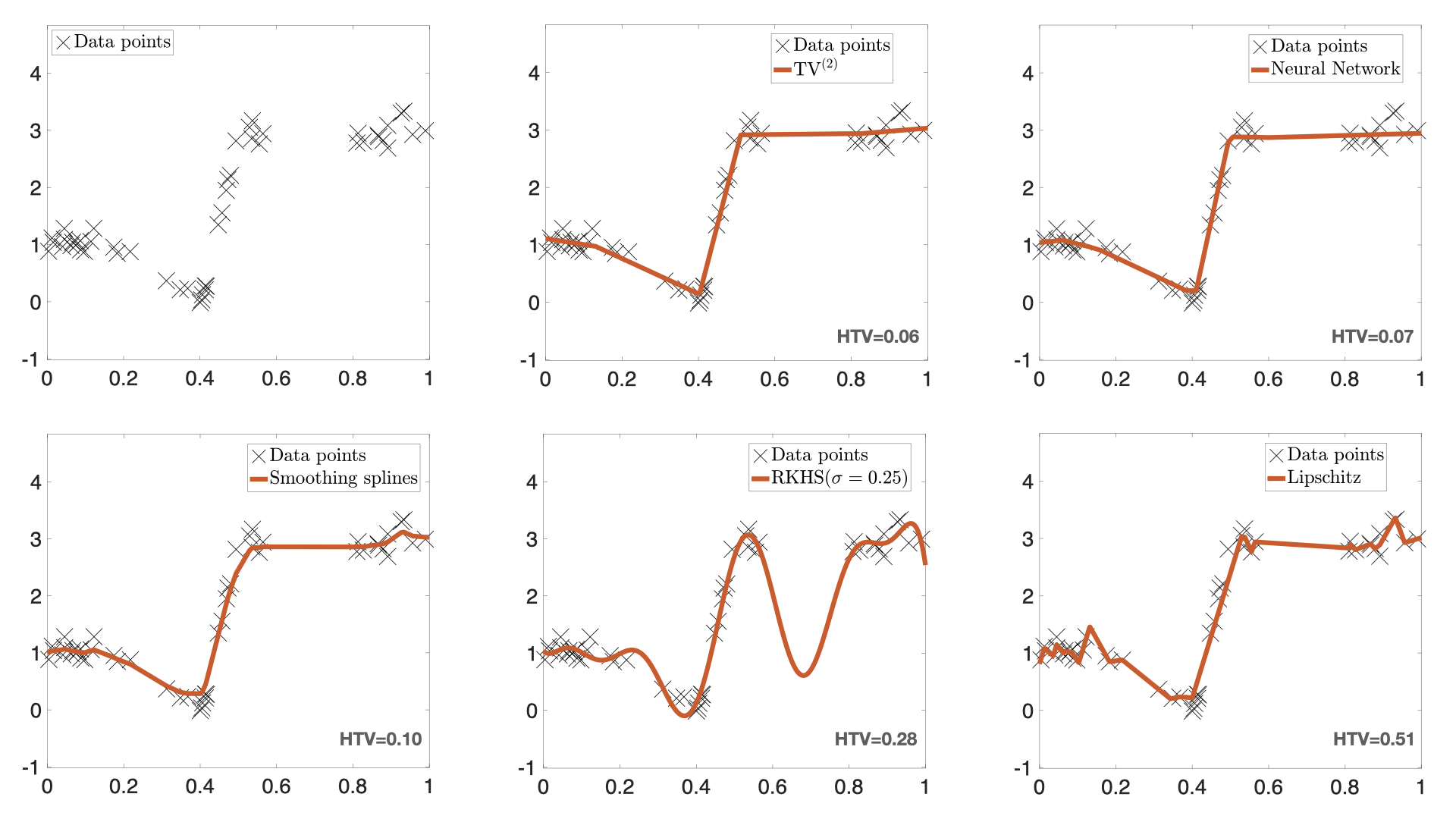}}
  \caption{  Comparison of five different learning schemes in the one-dimensional setting.}
  \label{Fig:1D} \medskip
\end{minipage}
\end{figure}
%

Next, we consider a 2D learning example  where we take $M=5000$ samples from a   2D height map obtained from a facial dataset\footnote{{https://www.turbosquid.com/3d-models/3d-male-head-model-1357522}}.  Note that there are gaps in the training data,  which makes the fitting problem more challenging.  In this case, we compare three different learning schemes: 
\begin{enumerate}
\item A ReLU neural network with 4 hidden layers, each layer consisting of 40 hidden neurons.
\item RKHS learning with a Gaussian radial-basis function   whose width is $\sigma=0.16$.
\item The framework of learning 2D functions with HTV regularization \cite{campos2021HTV}. 
\end{enumerate}
We tune the hyper-parameters of each framework to have a similar training error. The results are depicted in Figure \ref{Fig:2D}. Similarly to the previous case, this example highlights that the HTV favors simple and intuitive models that are visually more adequate.  
\begin{figure}[t]
\begin{minipage}{1.0\linewidth}
  \centering
  \centerline{\includegraphics[width=\linewidth]{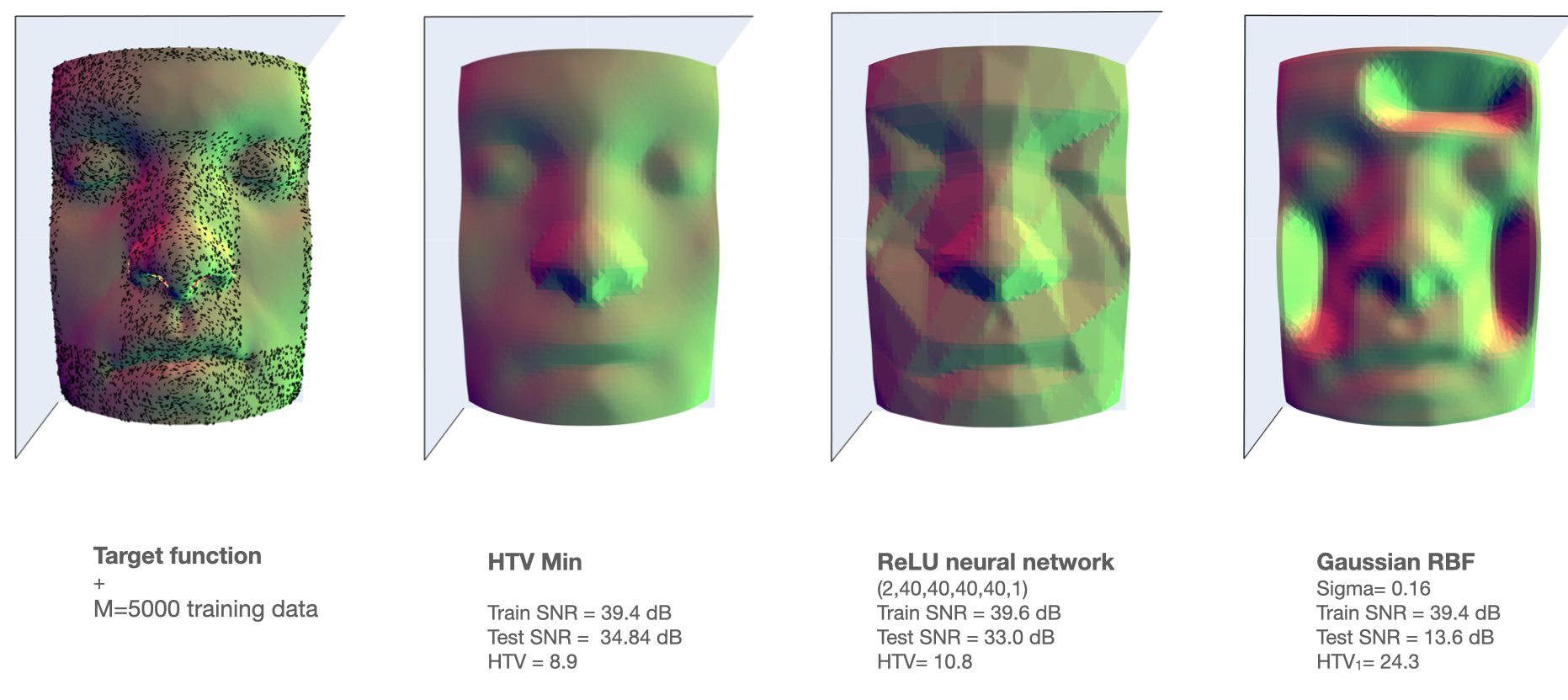}}
  \caption{Learning of a 2D height map of a face from its nonuniform samples. }
  \label{Fig:2D} \medskip
\end{minipage}
\end{figure}

Finally, we study the role of hyper-parameters in the complexity of the final learned mapping. To that end, we plot in Figure \ref{Fig:param_RBF} the ${\rm HTV}_p$ (for three different values of $p$) versus the regularization parameter $\lambda$ and the kernel width $\sigma$. As expected, sharper kernels and lower values of $\lambda$ correspond to a higher HTV in the output.  
\begin{figure}[t]
\begin{minipage}{1.0\linewidth}
  \centering
  \centerline{\includegraphics[width=\linewidth]{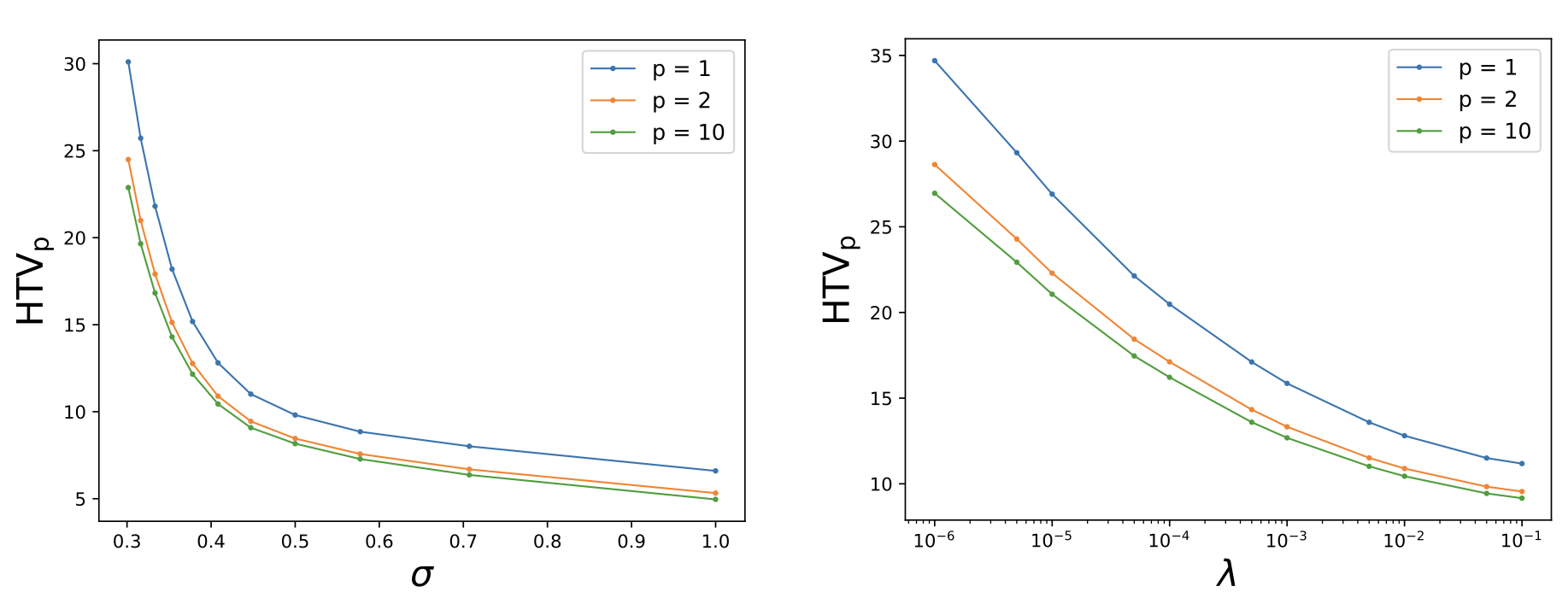}}
  \caption{The HTV of the learned mapping versus the regularization weight $\lambda$ (left) and the kernel width $\sigma$  (right) in the 2D example. }
  \label{Fig:param_RBF} \medskip
\end{minipage}
\end{figure}

 \section{Conclusion}
 In this paper, we have introduced the Hessian-Schatten total-variation (HTV) seminorm and proposed its use as a complexity measure for the  study of  learning schemes. Our notion of complexity is very general and can be applied to different scenarios. We have proven that the HTV enjoys the properties that are expected of a good complexity measure,  such as invariance to simple transformations and zero penalization of linear regressors. We then computed the HTV of two general classes of functions. In each case, we derived simple formulas for the HTV that allowed us to interpret its underlying behavior. Finally, we have provided some illustrative examples of usage for the comparison of learning algorithms.  Future research directions could be to use this notion of complexity to study learning schemes, in particular,  their generalization power.
 \appendix
 \section{Proof of Theorem \ref{Thm:SinfLinf}}\label{App:SinfLinf}
 \begin{proof}
    It is known that all norms are equivalent in finite-dimensional vector spaces. Consequently, there exist positive constants $c_1,c_2>0$ such that 
 \begin{equation*}
 \forall \mathbf{A}=[a_{i,j}]\in\mathbb{R}^{d\times d}, \quad c_1 \|\mathbf{A}\|_{\rm sum}\leq \|\mathbf{A}\|_{S_{q}} \leq  c_2 \|\mathbf{A}\|_{\rm sum},
 \end{equation*}
where $\|\mathbf{A}\|_{\rm sum}=\sum_{i=1}^d\sum_{j=1}^d |a_{i,j}| $. This immediately yields that
 \begin{equation}\label{Ineq1}
c_1 \sum_{i=1}^d \sum_{j=1}^d \|f_{i,j}\|_{L_\infty}\leq \|\mathbf{F}\|_{L_\infty,S_q} \leq c_2 \sum_{i=1}^d \sum_{j=1}^d \|f_{i,j}\|_{L_\infty},
 \end{equation}
as well as that
 \begin{equation}\label{Ineq2}
c_1 \sup_{\boldsymbol{x}\in\Omega}\|\mathbf{F}(\boldsymbol{x})\|_{\rm sum} \leq \|\mathbf{F}\|_{S_q,L_\infty} \leq c_2 \sup_{\boldsymbol{x}\in\Omega} \|\mathbf{F}(\boldsymbol{x})\|_{\rm sum},
 \end{equation}
for all $\mathbf{F}\in \mathcal{C}_0(\Omega;\mathbb{R}^{d\times d})$. On  the one hand, we have  that
\begin{equation}\label{Ineq3}
\sup_{\boldsymbol{x}\in\Omega}\|\mathbf{F}(\boldsymbol{x})\|_{\rm sum}= \sup_{\boldsymbol{x}\in\Omega} \left(\sum_{i,j=1}^d  |f_{i,j}(\boldsymbol{x})|\right)  \leq  \sum_{i,j=1}^d  \sup_{\boldsymbol{x}\in\Omega}|f_{i,j}(\boldsymbol{x})| = \sum_{i,j=1}^d  \|f_{i,j}\|_{L_\infty}.
\end{equation}
Combining \eqref{Ineq1}, \eqref{Ineq2} and \eqref{Ineq3}, we then deduce that 
\begin{equation}\label{Ineq:OneSide}
\|\mathbf{F}\|_{S_q,L_\infty} \leq c_2 \sup_{\boldsymbol{x}\in\Omega} \|\mathbf{F}(\boldsymbol{x})\|_{\rm sum} \leq c_2  \sum_{i,j=1}^d  \|f_{i,j}\|_{L_\infty} \leq \frac{c_2}{c_1} \|\mathbf{F}\|_{L_\infty,S_q}.
\end{equation}
On the other hand,  using $\|\mathbf{F}(\boldsymbol{x})\|_{\rm sum}\geq |f_{i,j}(\boldsymbol{x})|$ for all $i,j=1,\ldots,d$,  we obtain that 
\begin{equation*}
 \|f_{i,j}\|_{L_\infty} =  \sup_{\boldsymbol{x}\in\Omega}   |f_{i,j}(\boldsymbol{x})| \leq \sup_{\boldsymbol{x}\in\Omega}\|\mathbf{F}(\boldsymbol{x})\|_{\rm sum} , \quad \forall i,j=1,\ldots,d.
\end{equation*}
Summing over all $i,j=1,\ldots,d$ then gives that 
\begin{equation}\label{Ineq4}
\sum_{i=1}^d \sum_{j=1}^d \|f_{i,j}\|_{L_\infty} \leq d^{2}\sup_{\boldsymbol{x}\in\Omega}\|\mathbf{F}(\boldsymbol{x})\|_{\rm sum}.
\end{equation}
Combining \eqref{Ineq1}, \eqref{Ineq2},  and \eqref{Ineq4}, we obtain that 
\begin{equation}\label{Ineq:TheOtherSide}
 \|\mathbf{F}\|_{L_\infty,S_q} \leq  c_2\sum_{i=1}^d \sum_{j=1}^d \|f_{i,j}\|_{L_\infty}   \leq  c_2 d^2  \sup_{\boldsymbol{x}\in\Omega}\|\mathbf{F}(\boldsymbol{x})\|_{\rm sum}\leq  \frac{c_2}{c_1} d^2 \|\mathbf{F}\|_{S_q,L_\infty}.
\end{equation}
Finally, the inequalities \eqref{Ineq:OneSide} and \eqref{Ineq:TheOtherSide} yield \eqref{Ineq:NormEqv} with $A=\frac{c_1}{c_2}$ and $B=\frac{c_2}{c_1}d^2$ (Item 2). Further, it guarantees that the functional $\mathbf{F}\mapsto \|\mathbf{F}\|_{S_q,L_\infty}$ is well-defined (finite) for all $\mathbf{F}\in\mathcal{C}_0(\Omega,\mathbb{R}^{d\times d})$.  It is then easy to verify  the remaining norm properties (positivity, homogeneity and  the triangle inequality) of $\|\cdot\|_{S_q,L_\infty}$ (Item 1). As for Item 3, we note that the norm equivalence implies that both norms induce the same topology over $\mathcal{C}_0(\Omega;\mathbb{R}^{d\times d})$. Hence,  $\left(\mathcal{C}_0(\Omega;\mathbb{R}^{d\times d}),\|\cdot\|_{S_q,L_\infty}\right)$ is a {\it bona fide} Banach space.
\end{proof}
\section{Proof of Theorem \ref{Thm:RS1}}\label{App:RS1}
 \begin{proof}
{\bf Item 1:} We first show that the right-hand side of \eqref{Eq:DualNormEqForm} is well-defined and admits a finite value.  First, note that $\left\|\mathbf{W}(\cdot)\right\|_{S_p}$ is  the composition of the measurable function ${\bf W}:\Omega\rightarrow \mathbb{R}^{d\times d}$ and the Schatten-$p$ norm $\|\cdot\|_{S_p} :\mathbb{R}^{d\times d} \rightarrow \mathbb{R}$ that is continuous and, consequently, measurable. This implies that $\left\|\mathbf{W}(\cdot)\right\|_{S_p}$ is also a measurable function and, hence, its $L_1$ norm is well-defined.  The last step is to show that the $L_1$-norm is finite. From the norm-equivalence property of finite-dimensional vector spaces, we deduce the existence of $b>0$ such that,  for any ${\bf A}= [a_{i,j}]\in\mathbb{R}^{d\times d}$, we have that
\begin{equation}
\left\| {\bf A} \right\|_{S_p}\leq b \| {\bf A} \|_{\rm sum},
\end{equation}
where $\| {\bf A} \|_{\rm sum} = \sum_{i=1}^d \sum_{j=1}^d |a_{i,j}|$. This implies that
\begin{equation*}
\big\| \left\|\mathbf{W}(\cdot)\right\|_{S_p}  \big\|_{L_1} = \int_{\Omega}  \left\|\mathbf{W}(\boldsymbol{x})\right\|_{S_p} {\rm d} \boldsymbol{x} \leq b  \int_{\Omega}  \left\|\mathbf{W}(\boldsymbol{x})\right\|_{{\rm sum}} {\rm d} \boldsymbol{x}  \stackrel{{\rm (i)}}{=} b\sum_{i=1}^{d} \sum_{j=1}^d \|w_{i,j}\|_{L_1} < +\infty,
\end{equation*}
where we have used  Fubini's theorem to deduce (i).  Now, one readily verifies that 
\begin{align*}
\langle \mathbf{W}, \mathbf{F} \rangle &= \sum_{i,j=1}^d \langle w_{i,j} , f_{i,j}  \rangle  =  \sum_{i,j=1}^d \int_{\Omega} w_{i,j}(\boldsymbol{x}) f_{i,j}(\boldsymbol{x}) \mathrm{d}\boldsymbol{x}  =\int_{\Omega} \left(\sum_{i,j=1}^d  w_{i,j}(\boldsymbol{x}) f_{i,j}(\boldsymbol{x}) \right) \mathrm{d}\boldsymbol{x}   \\ &  \int_{\Omega} \left|\sum_{i,j=1}^d  w_{i,j}(\boldsymbol{x}) f_{i,j}(\boldsymbol{x}) \right| \mathrm{d}\boldsymbol{x}    \stackrel{{\rm (i)}}{\leq}   \int_{\Omega}    \| \mathbf{W}(\boldsymbol{x})\|_{S_p}  \|\mathbf{F}(\boldsymbol{x})\|_{S_{q}} \mathrm{d}\boldsymbol{x} \stackrel{{\rm (ii)}}{\leq}  \big\| \left\|\mathbf{W}(\cdot)\right\|_{S_p}  \big\|_{L_1} \|\mathbf{F}\|_{S_{q},L_\infty},
\end{align*}
where we have used the H\"older inequality for Schatten norms (see \eqref{Eq:Holder}) in (i) and the one for $L_p$ norms in (ii). We conclude that 
\begin{equation}
\|\mathbf{W}\|_{S_p,\mathcal{M}} \leq  \big\| \left\|\mathbf{W}(\cdot)\right\|_{S_p}  \big\|_{L_1}.
\end{equation}
To show the equality, we   need to prove that,  for any $\epsilon>0$, there exists an element $\mathbf{F}_{\epsilon} \in \mathcal{C}_0(\Omega;\mathbb{R}^{d\times d})$ with $ \|\mathbf{F}_\epsilon\|_{S_{q},L_\infty}=1$ such that 
\begin{equation}\label{Eq:OtherSide}
\langle \mathbf{W}, \mathbf{F}_\epsilon \rangle  \geq \big\| \left\|\mathbf{W}(\cdot)\right\|_{S_p}  \big\|_{L_1}- \epsilon.
\end{equation}
  Consider the function ${\bf F}: \Omega \rightarrow \mathbb{R}^{d\times d}$ with 
\begin{equation}
{\bf F}(\boldsymbol{x}) = \begin{cases}\frac{{\rm J}_{S_p,{\rm rank}}\left({\bf W}(\boldsymbol{x})\right)}{\| {\bf W}(\boldsymbol{x})\|_{S_p}}, & {\bf W}(\boldsymbol{x})\neq \boldsymbol{0}\\ 0, & \text{otherwise,}\end{cases}
 \end{equation}
 where ${\rm J}_{S_p,{\rm rank}}:\mathbb{R}^{d\times d}\rightarrow\mathbb{R}$ is the sparse duality mapping that maps ${\bf A}\in\mathbb{R}^{d\times d}$ to its minimum rank $(S_p,S_q)$-conjugate (see  \cite{aziznejad2020duality} for the definition and the proof of well-definedness)\footnote{This function coincides with the usual duality mapping for $p\in (1,+\infty)$ and the rank constraint is only needed for the  special case $p=1$.}.  We first note that  ${\bf F}$ is a measurable function. Indeed, from \cite{aziznejad2020duality}, we know that ${\rm J}_{S_p,{\rm rank}}$ is a measurable mapping over $\mathbb{R}^{d\times d}$. Hence, its composition with the measurable function ${\bf W}$ is also measurable.  Moreover, norms are continuous (and, so, Borel-measurable) functionals. Therefore, we have that  ${\bf F}(\boldsymbol{x}) = \mathbbm{1}_{{\bf W} \neq \boldsymbol{0}} \frac{{\rm J}_{S_p,{\rm rank}}\left({\bf W}(\boldsymbol{x})\right)}{\| {\bf W}(\boldsymbol{x})\|_{S_p}}$ is also Borel-measurable. Knowing the measurability of ${\bf F}$, we observe that 
 \begin{align}\label{Eq:Fsat}
 \int_{\Omega} {\rm Tr}( {\bf W}^T(\boldsymbol{x}) {\bf F}(\boldsymbol{x})) {\rm d} \boldsymbol{x} = \int_{\Omega}  \| \mathbf{W}(\boldsymbol{x})\|_{S_p}  \mathrm{d}\boldsymbol{x} = \big\| \left\|\mathbf{W}(\cdot)\right\|_{S_p}  \big\|_{L_1}.
\end{align}
  
We also note that $\|{\bf F}\|_{S_q,L_\infty}=1$. The final step is to use Lusin's theorem (see   \cite[Theorem 7.10]{folland1999real}) to find an $\epsilon$-approximation ${\bf F}_\epsilon \in \mathcal{C}_0(\Omega;\mathbb{R}^{d\times d})$  of ${\bf F}$ on the unit  $S_{q}-L_{\infty}$ ball so that 
 \begin{align}\label{Eq:LusinApp}
\left|  \int_{\Omega} {\rm Tr}( {\bf W}^T(\boldsymbol{x}) {\bf F}(\boldsymbol{x})) {\rm d} \boldsymbol{x} - \int_{\Omega} {\rm Tr}( {\bf W}^T(\boldsymbol{x}) {\bf F}_\epsilon(\boldsymbol{x})) {\rm d} \boldsymbol{x}\right| \leq \epsilon.
\end{align}
Now, combining  \eqref{Eq:LusinApp} with \eqref{Eq:Fsat}, we deduce \eqref{Eq:OtherSide} which completes the proof.

{\bf Item 2:} We first recall that the application of a distribution ${\bf D}$ of the form \eqref{Eq:WformDirac} to any element ${\bf F}\in \mathcal{C}_0(\Omega;\mathbb{R}^{d\times d})$ can be computed as 
\begin{equation}
\langle {\bf D}, {\bf F} \rangle = \int_{C} {\rm Tr}\left({\bf A}^T {\bf F}(\boldsymbol{x}, {\rm T} \boldsymbol{x} )\right) {\rm d}\boldsymbol{x}. 
\end{equation}
Using   H\"older's inequality, for any ${\bf F} \in  \mathcal{C}_0(\Omega;\mathbb{R}^{d\times d})$ with $\|{\bf F}\|_{S_q,L_\infty}=1$, we obtain that 
\begin{align*}
\int_{C} {\rm Tr}\left({\bf A}^T {\bf F}(\boldsymbol{x}, {\rm T} \boldsymbol{x})\right) {\rm d}\boldsymbol{x}&\leq \int_{C} \|{\bf A}\|_{S_p} \|{\bf F}(\boldsymbol{x}, {\rm T} \boldsymbol{x} )\|_{S_q} {\rm d}\boldsymbol{x} \\ & \leq \|{\bf A}\|_{S_p}  \int_{C}1{\rm d}\boldsymbol{x}=\|{\bf A}\|_{S_1} {\rm Leb}(C),
\end{align*}
which implies that $\|{\bf D}\|_{S_p,\mathcal{M}} \leq \|{\bf A}\|_{S_p} {\rm Leb}(C)$. To verify the equality, we   consider an element ${\bf F} \in \mathcal{C}_0(\Omega;\mathbb{R}^{d\times d})$ whose restriction on $C$ is the constant matrix ${\bf A}^* =  \|{\bf A}\|_{S_p}^{-1} {\rm J}_{S_p,{\rm rank}}({\bf A})$. 

{\bf Item 3:} Following the assumption that ${\rm Leb}(C_0)=0$, for any $\epsilon>0$, there exists a measurable set $E\subseteq \mathbb{R}^{d_1}$  with ${\rm Leb}(E)= \epsilon/2$ such that $C_0\subseteq E$. From the construction, we deduce that  the sets $C_1\backslash E$ and $C_2\backslash E$ are separable; hence, there exists a function ${\bf F}_{\epsilon}\in \mathcal{C}_0(\Omega;\mathbb{R}^{d\times d})$ with $\|{\bf F}_{\epsilon}\|_{S_{q},L_{\infty}}=1$ such that  
$${\bf F}_{\epsilon}(\boldsymbol{x}_1,{\bf T}_i\boldsymbol{x}_1) = {\bf A}_i^*, \quad \forall\boldsymbol{x}_1 \in C_i\backslash C_0, i=1,2,$$
where ${\bf A}_i^* = \|{\bf A}_i\|_{S_p}^{-1}{\rm J}_{S_p,{\rm rank}}({\bf A}_i),i=1,2$. This implies that,  for $i=1,2$,  we have that
\begin{align*}
 \langle {\bf D}_i, {\bf F}_{\epsilon} \rangle &=\int_{C_i} {\rm Tr}\left({\bf A}_i^T{\bf F}_{\epsilon}(\boldsymbol{x}_1,{\bf T}_i\boldsymbol{x}_1)\right) {\rm d}\boldsymbol{x}_1  
 \\&=\int_{  C_0} {\rm Tr}\left({\bf A}_i^T{\bf F}_{\epsilon}(\boldsymbol{x}_1,{\bf T}_i\boldsymbol{x}_1)\right) {\rm d}\boldsymbol{x}_1 +\int_{C_i\backslash C_0} {\rm Tr}\left({\bf A}_i^T{\bf A}_i^* \right) {\rm d}\boldsymbol{x}_1 
 \\ &\geq -{\rm Leb}(C_0) \|{\bf A}_i\|_{S_p} + {\rm Leb}(C_i\backslash C_0)  \|{\bf A}_i\|_{S_p} 
 \\& \geq \|{\bf A}_i\|_{S_p} ({\rm Leb}(C_i) - \epsilon).
\end{align*} 
Hence, for any $\epsilon>0$, we have that 
\begin{align*}
 \|{\bf D}_1 + {\bf D}_2\|_{S_p,\mathcal{M}} &\geq \langle {\bf D}_1 + {\bf D}_2 , {\bf F}_{\epsilon} \rangle 
\\&\geq \|{\bf A}_1\|_{S_p} {\rm Leb}(C_1)+\|{\bf A}_2\|_{S_p} {\rm Leb}(C_2) -  \epsilon(\|{\bf A}_1\|_{S_p}+\|{\bf A}_2\|_{S_p}).
\end{align*}
By letting $\epsilon\rightarrow 0$, we deduce that $ \|{\bf D}_1 + {\bf D}_2\|_{S_p,\mathcal{M}}  \geq \| {\bf D}_1\|_{S_p,\mathcal{M}} + \|{\bf D}_2\|_{S_p,\mathcal{M}}$ which, together with the triangle inequality, yields the announced equality. 
 \end{proof}
 \section{Proof of Theorem \ref{Thm:AI}}\label{App:AI}
 \begin{proof}

{\bf Item 1:} Starting from ${\rm H}\{f\}=\boldsymbol{0}$, we deduce that $ \frac{\partial^2 f}{\partial x_i^2}= 0$ for $i=1,\ldots,d$. Following Proposition 6.1 in \cite{unser2014anintroduction}, we deduce that the null space of $\frac{\partial^2}{\partial x_1^2 }$ can only contain (multivariate) polynomials. Using this, we infer that any $p$ in the null space of $\frac{\partial^2}{\partial x_1^2}$ is of the form   $ p(\boldsymbol{x}) = a_1 x_1 + q_1(\boldsymbol{x})$
for some $a_1\in\mathbb{R}$ and some multivariate polynomial $q_1$ that does not depend on $x_1$. Finally, one verifies by induction that $q_1(\boldsymbol{x}) = \sum_{i=2}^d a_i x_i + q_0(\boldsymbol{x})$, where $q_0$ is a multivariate polynomial that does not depend on any of its variables and so is constant, {\it i.e.} $q_0(\boldsymbol{x}) =b$ for some $b\in\mathbb{R}$. We conclude the proof by remarking that any affine mapping is indeed in the null space of $\rm H$. 
 
{\bf Item 2:}  By invoking that $ {\rm H}\{f(\cdot-\boldsymbol{x}_0) \} = {\rm H}\{f \}(\cdot-\boldsymbol{x}_0)$, we immediately deduce that 
\begin{align*}
\mathrm{HTV}\left(f(\cdot-\boldsymbol{x}_0)\right)&= \sup \left\{ \langle   {\rm H}\{f \}(\cdot-\boldsymbol{x}_0),\mathbf{F}\rangle:  \mathbf{F}\in  \mathcal{D}(\mathbb{R}^d;\mathbb{R}^{d\times d}), \|\mathbf{F}\|_{S_q,L_\infty}=1 \right\}\\& = \sup \left\{ \langle   {\rm H}\{f \},\mathbf{F}(\cdot+\boldsymbol{x}_0)\rangle:  \mathbf{F}\in  \mathcal{D}(\mathbb{R}^d;\mathbb{R}^{d\times d}), \|\mathbf{F}\|_{S_q,L_\infty}=1 \right\} \\&= \mathrm{HTV}(f).
\end{align*}
Similarly, following the chain rule, we obtain that $ {\rm H}\{f(\alpha\cdot) \} = \alpha^2 {\rm H}\{f \}(\alpha\cdot)$. This yields  that 
\begin{align*}
\mathrm{HTV}\left(f(\alpha\cdot)\right)&= \alpha^2\sup \left\{ \langle   {\rm H}\{f \}(\alpha\cdot),\mathbf{F}\rangle:  \mathbf{F}\in  \mathcal{D}(\mathbb{R}^d;\mathbb{R}^{d\times d}), \|\mathbf{F}\|_{S_q,L_\infty}=1 \right\}\\& = \alpha^2\sup \left\{ \langle   {\rm H}\{f \},\alpha^{-d}\mathbf{F}(\alpha^{-1}\cdot)\rangle:  \mathbf{F}\in  \mathcal{D}(\mathbb{R}^d;\mathbb{R}^{d\times d}), \|\mathbf{F}\|_{S_q,L_\infty}=1 \right\} \\& = |\alpha|^{2-d} \sup \left\{ \langle   {\rm H}\{f \}, \mathbf{F}( \cdot)\rangle:  \mathbf{F}\in  \mathcal{D}(\mathbb{R}^d;\mathbb{R}^{d\times d}), \|\mathbf{F}\|_{S_q,L_\infty}=1 \right\}\\&= |\alpha|^{2-d}\mathrm{HTV}(f).
\end{align*}
As for the last invariance property, we use the formula for the  Hessian of a rotated function 
$${\rm H}\{f({\bf U}\cdot) \} = {\bf U}^T {\rm H}\{f \}({\bf U}\cdot) {\bf U}.$$
This implies that
\begin{align*}
\mathrm{HTV}\left(f({\bf U}\cdot)\right)&=  \sup \left\{ \langle {\bf U}^T  {\rm H}\{f \}({\bf U} \cdot){\bf U},\mathbf{F}\rangle:  \mathbf{F}\in  \mathcal{D}(\mathbb{R}^d;\mathbb{R}^{d\times d}), \|\mathbf{F}\|_{S_q,L_\infty}=1 \right\}\\& =  \sup \left\{ \langle   {\rm H}\{f \}({\bf U}  \cdot), {\bf U} \mathbf{F}(\cdot){\bf U}^T\rangle:  \mathbf{F}\in  \mathcal{D}(\mathbb{R}^d;\mathbb{R}^{d\times d}), \|\mathbf{F}\|_{S_q,L_\infty}=1 \right\} \\& =  \sup \left\{ \langle   {\rm H}\{f \}, {\bf U}\mathbf{F}({\bf U}^T \cdot){\bf U}^T\rangle:  \mathbf{F}\in  \mathcal{D}(\mathbb{R}^d;\mathbb{R}^{d\times d}), \|\mathbf{F}\|_{S_q,L_\infty}=1 \right\}\\&=   \mathrm{HTV}(f),
\end{align*}
where the last equality follows from the invariance of Schatten norms under orthogonal transformations (as exploited, for example, in \cite{Lefki2012Hessian,Lefki2013Hessian}).
\end{proof}
\section{Proof of Theorem \ref{Thm:CPWLcalc}}\label{App:CPWL}
Let $f:\mathbb{R}^d\rightarrow\mathbb{R}$ be a CPWL function with linear regions $P_n \subseteq \Omega$ and affine parameters $\boldsymbol{a}_n \in\mathbb{R}^d$ and $b_n \in \mathbb{R}$, for $n=1,\ldots,N$.  We first compute the gradient of $f$. 
\begin{lemma}\label{Lem:CPWLgrad}
The gradient of a CPWL function $f:\Omega \rightarrow\mathbb{R}$ as described above can be expressed as 
\begin{equation}\label{Eq:CPWLgrad}
\boldsymbol{\nabla} f(\boldsymbol{x})= \sum_{n=1}^N \boldsymbol{a}_n \mathbbm{1}_{P_n}(\boldsymbol{x}), 
\end{equation}
for almost every $\boldsymbol{x}\in\Omega$.
\end{lemma}
\begin{proof}
The interior of $P_n$  is denoted by $U_n$  with $n=1\ldots,N$. We then note that $\Omega\backslash \left( \bigcup_{n=1}^N U_n\right) $ is a set of measure zero. Hence, it is sufficient to show that  $\boldsymbol{\nabla} f(\boldsymbol{x})=\boldsymbol{a}_n$ for any $\boldsymbol{x}_0=(x_{0,1},\ldots,x_{0,d})\in U_n$. We define the functions $g_i:\mathbb{R}\rightarrow\mathbb{R}$ as 
$$g_i(x) = f(x_{0,1}, \ldots, x_{0,i-1}, x, x_{0,i+1},\ldots,x_{0,d}).$$
Following the definition of CPWL mappings, $g_i$ is a linear spline ({\it i.e.}, a 1D continuous and piecewise-linear function). Hence, it is locally linear and can be expressed as $g_i(x) = a_{n,i} x + (\sum_{j\neq i} a_{n,j} x_{0,j} +b)$  in an open neighborhood of $x_{0,i}$. Moreover, it is clear  that $a_{n,i}= g_i'(x_{0,i})= \frac{\partial f}{\partial x_i}(\boldsymbol{x}_0)$. Hence, 
$$ \boldsymbol{\nabla} f(\boldsymbol{x}_0)= \left( \frac{\partial f}{\partial x_1}(\boldsymbol{x}_0),\ldots,\frac{\partial f}{\partial x_d}(\boldsymbol{x}_0)\right) = \left( a_{n,1},\ldots,a_{n,d}\right) = \boldsymbol{a}_n.$$
\end{proof}
\begin{proof}[Proof of Theorem  \ref{Thm:CPWLcalc}]
We start by introducing some notions that are required in the proof. For each $n=1,\ldots,N$ and $k\in {\rm adj}_n$, we denote the intersection of $P_n$ and $P_k$ by $L_{n,k}=P_n \cap P_k$, which is itself a convex polytope with co-dimension $(d-1)$,  in the sense that it lies on a hyperplane $H_{n,k} = \{\boldsymbol{x}\in\mathbb{R}^d: \boldsymbol{u}_{n,k}^T \boldsymbol{x} + \beta_{n,k}=0\}$ for some normal vector $\boldsymbol{u}_{n,k}= (u_{n,k,i})\in \mathbb{R}^d$ with $\|\boldsymbol{u}_{n,k}\|_2=1$ and some shift value $\beta_{n,k}\in \mathbb{R}$. We adopt the convention that $\boldsymbol{u}_{n,k}$ refers to the outward normal vector, so that $\boldsymbol{u}_{n,k}^T\boldsymbol{x}+ \beta_{n,k} \leq 0$ for all $\boldsymbol{x}\in P_n$. We divide the proof in four steps:

{\bf Step 1: Transformation to the General Position.} First,  without any loss of generality,  we assume that all entries of $\boldsymbol{u}_{n,k}$ for all $n=1,\ldots,N$ and $k\in{\rm adj}_n$ are nonzero. Consider a unitary matrix ${\bf V}\in\mathbb{R}^{d\times d}$ such that $[{\bf V} \boldsymbol{u}_{n,k}]_i \neq 0$ for all $n=1,\ldots,N$, $k\in{\rm adj}_n$,  and $i=1,\ldots,d$. We remark that the function $g= f({\bf V}\cdot)$ is CPWL with linear regions $\tilde{P}_n= {\bf V}^T P_n$ and affine parameters $\tilde{\boldsymbol{a}}_n = {\bf V}^T\boldsymbol{a}_n$ and $\tilde{b}_n=b_n$ for $n=1,\ldots,N$. Now, if \eqref{Eq:CPWLcalc} holds for $g$, then we can invoke the invariance properties of the HTV (see Theorem \ref{Thm:AI}) to deduce that  
\begin{align*}
\mathrm{HTV}_p(f) &= \mathrm{HTV}_p\left(g\right) 
\\& =   \frac{1}{2}\sum_{n=1}^N \sum_{k\in {\rm adj}_n} \|\tilde{\boldsymbol{a}}_{n}-\tilde{\boldsymbol{a}}_{k}\|_2 H^{d-1}(\tilde{P}_n \cap \tilde{P}_k)
\\& = \frac{1}{2}\sum_{n=1}^N \sum_{k\in {\rm adj}_n} \|{\bf V}^T (\boldsymbol{a}_{n}-\boldsymbol{a}_{k})\|_2 H^{d-1}({\bf V}^T(P_n \cap P_k))
 \\ &  = \frac{1}{2}\sum_{n=1}^N \sum_{k\in {\rm adj}_n} \|\boldsymbol{a}_{n}-\boldsymbol{a}_{k}\|_2 H^{d-1}(P_n \cap P_k),
\end{align*}
where the last equality is due to the invariance of the Hausdorff measure and the $\ell_2$ norm to orthonormal transformations. 

{\bf Step 2: Calculation of the Hessian Distribution.} From now on, we assume that all entries of $u_{n,k}$ are nonzero,  with
$$ u_{n,k,i}=0 , \qquad n=0,\ldots,N, \quad k\in {\rm adj}_n,\quad i=1,\ldots,d.$$
This allows us to view $H_{n,k}$ as the graph of the affine mapping $T_{n,k}:\mathbb{R}^{d-1}\rightarrow\mathbb{R}$,  with 
$$ T_{n,k}(x_1,\ldots,x_{d-1})= \beta_{n,k}- \frac{\sum_{i=1}^{d-1} u_{n,k,i}x_i }{u_{n,k,d}},$$
and to define $C_{n,k}= \{ \boldsymbol{x}\in\mathbb{R}^{d-1}: (\boldsymbol{x}, T_{n,k} \boldsymbol{x}) \in L_{n,k}\}\subseteq\Omega$ as the preimage of $L_{n,k}$ over $T_{n,k}$. We also remark that, due to this affine projection, the  $(d-1)$-dimensional  Hausdorff measure of $L_{n,k}$ and the Lebesgue measure of $C_{n,k}$ are related by the coefficient $u_{n,k,d}$.  Indeed, we have that $H^{d-1}(L_{n,k}) = \frac{{{\rm Leb}}(C_{n,k})}{|u_{n,k,d}|}$. Using these notions, we now compute the matrix-valued distribution ${\rm H}\{f\} \in\mathcal{M}(\Omega;\mathbb{R}^{d\times d})$. We first note that, for all  $n=0,\ldots,N$ and $i=1,\ldots,d$, we have that 
\begin{equation}
\frac{\partial \mathbbm{1}_{P_n}}{\partial x_i}(\boldsymbol{x})= \sum_{k\in {\rm adj}_n}  -{\rm sgn}(u_{n,k,i}) \delta\left(x_i + \frac{\sum_{j\neq i} u_{n,k,j} x_j + \beta_{n,k}}{u_{n,k,i}} \right) \mathbbm{1}_{L_{n,k}}(\boldsymbol{x}).
\end{equation}
 Using the relation $\delta(\alpha \cdot) = |\alpha|^{-1} \delta(\cdot)$ for all $\alpha\in \mathbb{R}$, we obtain that 
\begin{equation}
 \frac{\partial \mathbbm{1}_{P_n}}{\partial x_i}(\boldsymbol{x})= \sum_{k\in {\rm adj}_n} \frac{-u_{n,k,i}}{|u_{n,k,d}|}\delta\left(x_d - T_{n,k} \boldsymbol{x}_1 \right)\mathbbm{1}_{L_{n,k}}(\boldsymbol{x}),
\end{equation}
where $\boldsymbol{x}_1=(x_1,\ldots,x_{d-1})\in\mathbb{R}^{d-1}$. Following the definition of $C_{n,k}$, we immediately get that 
\begin{equation}
\delta\left(x_d -   T_{n,k} \boldsymbol{x}_1\right)\mathbbm{1}_{L_{n,k}}(\boldsymbol{x})=\delta\left(x_d -  T_{n,k} \boldsymbol{x}_2 \right)\mathbbm{1}_{C_{n,k}}(\boldsymbol{x}_1),
\end{equation}
which   leads to 
\begin{equation}\label{Eq:partialInd}
\frac{\partial \mathbbm{1}_{P_n}}{\partial x_i}(\boldsymbol{x})= \sum_{k\in {\rm adj}_n} \frac{-u_{n,k,i}}{|u_{n,k,d}|}\delta\left(x_d - T_{n,k} \boldsymbol{x}_1 \right) \mathbbm{1}_{C_{n,k}}(\boldsymbol{x}_1).
\end{equation}
Combining \eqref{Eq:partialInd} with Lemma \ref{Lem:CPWLgrad}, we then deduce that 
\begin{align*}
 \frac{\partial^2 f}{\partial x_i\partial x_j} (\boldsymbol{x})&= \sum_{n=1}^N {a}_{n,j} \frac{\partial \mathbbm{1}_{P_n}}{\partial x_i}(\boldsymbol{x})\\&= \sum_{n=1}^N {a}_{n,j} \sum_{k\in {\rm adj}_n} \frac{-u_{n,k,i}}{|u_{n,k,d}|}\delta\left(x_d - T_{n,k} \boldsymbol{x}_1 \right) \mathbbm{1}_{C_{n,k}}(\boldsymbol{x}_1).
\end{align*}
Now, since $L_{n,k} = P_n\cap P_{k}$ and   $\boldsymbol{u}_{n,k} =( - \boldsymbol{u}_{k,n})$, we can rewrite the second-order partial derivatives as 
 \begin{align*}
 \frac{\partial^2 f}{\partial x_i\partial x_j} (\boldsymbol{x})&= \frac{1}{2}\sum_{n=1}^N \sum_{k\in {\rm adj}_n} ({a}_{k,j}-a_{n,j}) 
 \frac{u_{n,k,i}}{|u_{n,k,d}|}\delta\left(x_d - T_{n,k} \boldsymbol{x}_1 \right) \mathbbm{1}_{C_{n,k}}(\boldsymbol{x}_1).
 \end{align*}
Putting it in  matrix form, we conclude that  the Hessian is a sum of disjoint  Dirac fences,  as in  
\begin{align}\label{Eq:HessianCPWL}
{\rm H}\{f\}(\boldsymbol{x}) & = \left[ \frac{\partial^2 f}{\partial x_i\partial x_j}(\boldsymbol{x}) \right] \nonumber \\
& = \frac{1}{2}\sum_{n=1}^N \sum_{k\in {\rm adj}_n} \left[({a}_{k,j}-a_{n,j}) 
 \frac{u_{n,k,i}}{|u_{n,k,d}|}\right]\delta\left(x_d - T_{n,k} \boldsymbol{x}_1 \right) \mathbbm{1}_{C_{n,k}}(\boldsymbol{x}_1).
\end{align}
{\bf Step 3: Computation of the HTV.}  By invoking Item 3 of Theorem \ref{Thm:RS1}, we deduce that 
\begin{align}
\mathrm{HTV}_p(f) &= \left\| {\rm H}\{f\} \right\|_{S_p,\mathcal{M}}  \nonumber
\\& =  \frac{1}{2}\sum_{n=1}^N \sum_{k\in {\rm adj}_n} \left\|\left[({a}_{k,j}-a_{n,j}) 
 \frac{u_{n,k,i}}{|u_{n,k,d}|}\right]\delta\left(x_d - T_{n,k} \boldsymbol{x}_1 \right) \mathbbm{1}_{C_{n,k}}(\boldsymbol{x}_1)\right\|_{S_p,\mathcal{M}} \nonumber
\\& = \frac{1}{2}\sum_{n=1}^N \sum_{k\in {\rm adj}_n} \left\|\left[({a}_{k,j}-a_{n,j}) 
 \frac{u_{n,k,i}}{|u_{n,k,d}|}\right]\right\|_{S_p} {\rm Leb}(C_{n,k}), \label{Eq:FinalRSHCPWL}
\end{align}
where the last equality results from  Item 2 of Theorem \ref{Thm:RS1}. 

Finally, we use the continuity of $f$ to deduce that,  for any pair of points $\boldsymbol{p}_1,\boldsymbol{p}_2\in H_{n,k}$, we have that 
 \begin{align*}
 \boldsymbol{a}_n^T \boldsymbol{p}_i + b_n =  \boldsymbol{a}_k^T \boldsymbol{p}_i + b_k, \qquad i=1,2.  
 \end{align*}
 Subtracting the above equalities for $i=1$ and $i=2$,  we obtain that 
 \begin{align*}
 \boldsymbol{a}_n^T (\boldsymbol{p}_1-\boldsymbol{p}_2) =  \boldsymbol{a}_k^T (\boldsymbol{p}_1-\boldsymbol{p}_2).  
 \end{align*}
 However, $(\boldsymbol{p}_1-\boldsymbol{p}_2)$ is orthogonal to $\boldsymbol{u}_{n,k}$. Hence, the vector $(\boldsymbol{a}_k-\boldsymbol{a}_n)$ points in the direction of $\boldsymbol{u}_{n,k}$. This implies that the matrix 
 $$\left[({a}_{k,j}-a_{n,j}) 
 \frac{u_{n,k,i}}{|u_{n,k,d}|}\right]= |u_{n,k,d}|^{-1} \boldsymbol{u}_{n,k}(\boldsymbol{a}_k- \boldsymbol{a}_n)^T =\frac{\|\boldsymbol{a}_k-\boldsymbol{a}_n\|_2}{|u_{n,k,d}|}  \boldsymbol{u}_{n,k}\boldsymbol{u}_{n,k}^T$$ 
 is rank-1 and symmetric. Hence, for any  $p\in [1,+\infty]$, its Schatten-$p$ norm is equal to the absolute value of its trace.   The replacement of this in \eqref{Eq:FinalRSHCPWL} and the use of $H^{d-1}(L_{n,k}) = \frac{{{\rm Leb}}(C_{n,k})}{|u_{n,k,d}|}$ yields the announced expression \eqref{Eq:CPWLcalc}.
 \end{proof}
\bibliographystyle{IEEEbib.bst}
\bibliography{Aziznejad.bib}
\end{document}